\documentclass[10pt,twocolumn,letterpaper]{article}

\usepackage[options]{cvpr}      

\usepackage{amsthm}
\newtheorem{theorem}{Theorem} 
\definecolor{cvprblue}{rgb}{0.21,0.49,0.74}
\usepackage[pagebackref,breaklinks,colorlinks,allcolors=cvprblue]{hyperref}

\usepackage{amsthm}
\usepackage{booktabs}
\usepackage{siunitx}
\usepackage{natbib}
\usepackage{tikz}
\usepackage{pgfplots}
\pgfplotsset{compat=1.18} 
\usepackage{graphicx}
\usepackage{subcaption} 
\usetikzlibrary{arrows.meta,positioning,fit,calc}
\usepackage{caption}
\usepackage{tcolorbox}
\usepackage{fancyhdr}
\AtBeginDocument{%
  \setlength{\abovedisplayskip}{4pt plus 2pt minus 1pt}%
  \setlength{\belowdisplayskip}{4pt plus 2pt minus 1pt}%
  \setlength{\abovedisplayshortskip}{2pt plus 1pt minus 1pt}%
  \setlength{\belowdisplayshortskip}{2pt plus 1pt minus 1pt}%
}
\def\paperID{19} 
\def\confName{CVPR}
\def\confYear{2026}

\title{From Actions to Understanding: Conformal Interpretability of Temporal Concepts in LLM Agents}

\author{
 \textbf{Ramneet Kaur* \textsuperscript{2}},
 \textbf{Trilok Padhi* \textsuperscript{1}},
 \textbf{Krishiv Agarwal \textsuperscript{4}},
 \textbf{Adam D. Cobb\textsuperscript{2}},\\
 \textbf{Daniel Elenius\textsuperscript{2}},
  \textbf{Manoj Acharya\textsuperscript{2}},
  \textbf{Colin Samplawski\textsuperscript{2}},
 \textbf{Alexander M. Berenbeim\textsuperscript{3}},\\
 \textbf{Nathaniel D. Bastian\textsuperscript{3}},
 \textbf{Susmit Jha\textsuperscript{2}},
 \textbf{Ugur Kursuncu\textsuperscript{1}},
 \textbf{Anirban Roy\textsuperscript{2}}
\\
    \textsuperscript{*}These authors contributed equally \\
 \textsuperscript{1}Georgia State University, Atlanta, USA \\
 \textsuperscript{2}Computer Science Lab, SRI, Menlo Park, USA \\
 \textsuperscript{3}Robotics Research Center, United States Military Academy, West Point, NY USA\\
 \textsuperscript{4}Computer Science Department, University of Florida, USA \\
\\
 \small{
   \textbf{Correspondence:} \href{mailto:email@domain}{tpadhi1@student.gsu.edu, ramneet.kaur@sri.com}
 }
}

\begin{document}
\maketitle
\begin{abstract}
        Large Language Models (LLMs) are increasingly deployed as autonomous agents capable of reasoning, planning, and acting within interactive environments. Despite their growing capability to perform multi-step reasoning and decision-making tasks, internal mechanisms guiding their sequential behavior remain opaque. This paper presents a framework for interpreting the temporal evolution of concepts in LLM agents through a step-wise conformal lens. We introduce the \textit{conformal interpretability framework for temporal tasks}, which combines step-wise reward modeling with conformal prediction to statistically label model's internal representation at each step as successful or failing. Linear probes are then trained on these representations to identify directions of temporal concepts—latent directions in the model's activation space that correspond to consistent notions of success, failure or reasoning drift. Experimental results on two simulated interactive environments, namely ScienceWorld and AlfWorld, demonstrate that these temporal concepts are linearly separable, revealing interpretable structures aligned with task success. We further show preliminary results on improving an LLM agent's performance by leveraging the proposed framework for steering the identified successful directions inside the model. The proposed approach, thus, offers a principled method for early failure detection as well as intervention in LLM-based agents, paving the path towards trustworthy autonomous language models in complex interactive settings.
    
  
\end{abstract}  
\thispagestyle{fancy}
\section{Introduction}
\label{sec:intro}

\begin{figure*}[t]
    \centering
    \includegraphics[width=\textwidth]{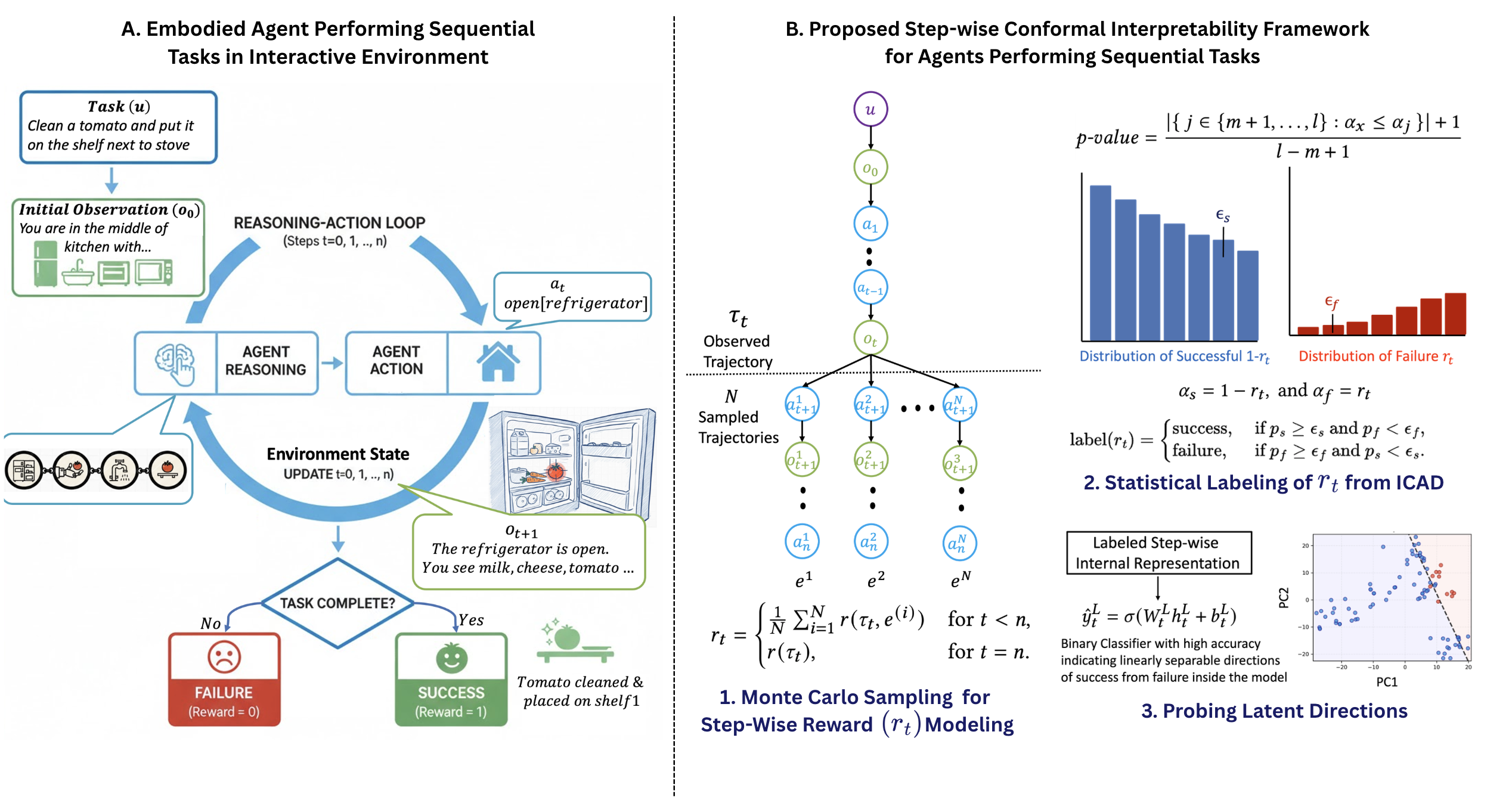}
    \caption{\footnotesize{
    \textbf{A.} We consider the problem of temporal interpretability of LLM agents trained to perform sequential tasks in complex environments.
    \textbf{B.} Proposed framework combines step-wise reward modeling with conformal labeling to distinguish success and failure at each timestep. Linear probes are trained on the model’s internal representations to test the hypothesis that step-wise notions of success and failure are linearly separable within the representation space.
    }}
    \label{fig:paper_idea}
\end{figure*}


Large Language Models (LLMs) have rapidly evolved from static text generators~\citep{gpt, llama-2} to autonomous agents capable of 
navigating, planning, and acting in the physical world~\citep{wangconformalnl2ltl, ricl, llm_agent1}.
When deployed in embodied simulators such as Alfworld~\citep{alfworld}, and ScienceWorld~\citep{sciworld}, these agents exhibit impressive decision-making abilities on sequential goals for accomplishing the assigned task~\citep{planning_affordances}. However, the reasoning process behind such sequential behavior remains largely opaque as these LLM agents continue to operate as \textit{black box}, leaving practitioners without reliable explanations for \textit{why} a model succeeds or \textit{where} it fails within the task trajectory. This has led to increasing interest in the \textit{interpretability} of LLMs, aimed at uncovering how models process inputs to generate outputs in a manner that is transparent and understandable to humans.

Traditional interpretability frameworks such as attribute or feature visualization~\citep{attribute, shap}, mechanistic interpretability via sparse autoencoders~\citep{sae}, activation space analysis via representation engineering~\citep{repe} and universal steering~\citep{us} offer valuable insights into mappings between \textit{standalone input} and the model's latent space. These approaches are, however, fundamentally limited in capturing \textit{temporal dynamics} within LLM-based agents as they operate on single standalone inputs such as a image, text or an image-text pair. While performing a sequential task
in an interactive environment, the agent’s decision at any step depends not only on the current observation but also on the entire trajectory history with prior reasoning traces, environment responses, and accumulated context. 
Model's internal representation space evolves through time, encoding both correct and incorrect reasoning directions. Understanding these evolving representations requires a temporally aware interpretability framework interpretability.

A straightforward approach to adapting the existing interpretability frameworks for temporal dynamics would be treating the entire trajectory as a single input and performing analysis on those. This approach, however, does not provide any granularity on \textit{which intermediate step} of the agent contributed to success or failure of the assigned task. An agent that successfully completes the task after ten steps hides the possibility that six of those might be suboptimal or irrelevant. Without a step-level notion of ``success'', it becomes impossible to pinpoint when the model’s internal dynamics shift from aligned reasoning to failure modes such as hallucination or irrelevant exploration.

A motivating example of the Llama-2~\citep{llama-2} agent trained to perform household tasks in the AlfWorld environment is as follows. For the assigned task of \textit{``placing a roll of paper next to the toilet''}, 
we observed that the agent starts on the right track  by moving towards the toilet and searching shelves to locate the roll. Yet midway through the trajectory, its internal reasoning begins to drift. Instead of continuing the search, the agent hallucinates that a non-existent drawer contains the paper roll, issuing an invalid action sequence. This behavioral deviation—despite earlier correct reasoning—leads to task failure, even though the initial steps were coherent and purposeful.


Such examples (more shown in Fig.~\ref{fig:sw_steering_results}) highlight a central challenge in \textit{temporal interpretability}: the final success or failure of a task often obscures \textit{where} in the sequence the agent’s internal representations began to diverge. A trajectory that ultimately fails may contain multiple locally successful steps, while an apparently successful run may involve accidental recoveries after erroneous decisions. Current interpretability methods—focused on isolated prompts, static embeddings, or single-step explanations—cannot capture these evolving internal transitions, nor can they identify \textit{when} the model’s reasoning begins to deviate.

This work introduces a \textit{step-wise conformal interpretability framework} that transforms the problem of agent's final evaluation on a multi-step sequential task into \textit{temporal representation analysis}. 
We posit that the \textit{directions of success and failure are geometrically separable} in an LLM fine-tuned to perform sequential tasks, .i.e, these models internally have separable notions of success and failure. 
To test this hypothesis, we build the proposed interpretability framework on the following three key components (Fig.\ref{fig:paper_idea}(B)):\\
\textbf{1. Step-Wise Reward Modeling:} We follow~\citet{ipr}'s approach on generating fine-grained step-wise rewards using Monte Carlo sampling over future trajectories. This transforms sparse final rewards into dense temporal feedback signals on success (or failure) of every step.\\
\textbf{2. Statistical Labeling via Conformal Prediction:} We propose using Inductive Conformal Prediction (ICP) to label model's internal representation at each step as successful or failing with provable confidence bounds.\\
\textbf{3. Probing Latent Directions:} We train classifiers (or linear probes) on layer and time-conditioned agent's activations to distinguish the latent space of success from failure. 

\textbf{Key Findings.} High accuracy and F1-scores of linear probes on Llama2-7B in two complex simulated environments, namely ScienceWorld~\citep{sciworld} and AlfWorld~\citep{alfworld}, validate our hypothesis that an LLM fine-tuned for sequential tasks develops an internal notion of linearly separable step-wise success across (a) timesteps, (b) layers, and (c) domains: both in-distribution and out-of-distribution.

\textbf{Steering the model towards Success.} As observed in prior work~\citep{shinnreflexion, wu2023brief} and validated in our experiments, these agents lack the intrinsic ability to self-correct themselves back towards the success directions and often continue drifting along failure trajectories once the deviation occurs. 
The proposed framework lays the foundation for \textit{steering} LLM agents towards the right direction by enabling targeted interventions when early signs of hallucination or misalignment emerge in their internal representations. We conduct preliminary experiments on steering the model towards its (identified) internal `success' directions and observe that the accuracy of the steered model improves from the baseline model, providing evidence on the practical use of the proposed framework.

\section{Related Work}
\label{sec:rel}
\textbf{LLM Agents for Interactive Embodied Environments.}
LLMs have evolved beyond text generation to function as powerful policy models for decision-making in interactive environments. Early systems like WebGPT~\citep{nakano2021webgpt} and SimpleTOD~\citep{hosseini2020simpletod} used human feedback or annotated dialogues for interactive learning, while ReAct~\citep{yao2022react} demonstrated a more scalable approach by integrating natural language reasoning directly into the decision loop, removing human from the loop. Building on this paradigm, works such as IPR~\citep{ipr}, SayCan~\citep{ahn2022saycan} and Inner Monologue~\citep{huang2022innermonologue} extend LLM-based decision-making to embodied and interactive domains, where agents reason over both language and environmental feedback. We perform interpretability analysis on LLMs fine-tuned using ReACT style to autonomously perform sequential tasks in interactive environments. 


\textbf{Interpretability of Large Language Models (LLMs).}
Existing interpretability methods include attribution techniques (e.g., saliency maps and masking-based perturbations) that relate outputs to input tokens~\citep{saliency_map,masking_explainity}, and post-hoc methods where the model explains or verbalizes its reasoning~\citep{self_explain,cot_reasoning}. These approaches, while broadly applicable, provide limited insight into the mechanisms underlying model behavior. Mechanistic interpretability instead analyzes internal activations, either at the level of individual neurons or distributed features~\citep{sae,repe,us}. However, most work is static, associating standalone inputs with concepts and ignoring how such concepts evolve over time as models reason and act. Time-series interpretability for LLMs has been explored primarily in domains such as weather, finance, and health~\citep{interp_time_series_forecasting,interp_time_series_classification}, but not in the interactive agentic settings we study.

\textbf{Use of Probes for Interpretability Analysis.}
Research on probes in LLMs has developed into a major branch of linguistic and interpretability analysis. Early studies~\citep{conneau2018cram, hewitt2019structural},  used linear probes to show that intermediate layers of models like BERT~\citep{tenney2019bert} encode rich linguistic structures, including parts of speech, syntactic trees, and semantic roles. Foundational work by \citet{alain2017understanding} introduced the concept of auxiliary classifiers as probes, inspiring subsequent analysis of internal representations in neural networks. Recent work shows that truthful answers to factual questions correspond to approximately linear directions in activation space~\citep{geom_of_truth}. We similarly use linear probes on hidden states of an LLM agent to test whether success vs.\ failure during sequential tasks is encoded along separable directions over time, using probe accuracy as a proxy for this separability.

\textbf{Conformal Prediction in Explainable AI.}
Conformal prediction (CP) provides distribution-free uncertainty quantification via calibrated prediction sets with guaranteed coverage~\citep{cp}. Recent work has begun to apply CP to LLMs for uncertainty estimation, out-of-distribution detection, safety, and evaluation~\citep{addressing_uncertainty_llms,polysemantic_dropout,prune_n_predict,safepath,llm_judge_uncertainty,conformal_language_modeling}. In contrast, we use CP in a novel way: to assign calibrated success/failure labels to temporal representations of an LLM agent, yielding bounded error rates when constructing trajectory-level interpretability signals.
\section{Background}
\label{sec:back}
\subsection{Performing Sequential Tasks in Interactive Environment}

\paragraph{Problem Setting:} As shown in Fig.~\ref{fig:paper_idea} (A), we consider a general class of sequential decision-making tasks for LLM agents in an interactive environment. Specifically, an LLM interacts with the environment by receiving textual observations in response to each action performed by the model. The goal of the agent is to execute a coherent sequence of steps that lead to the successful task completion while reasoning and acting at each step through natural language. 
\vspace{-2mm}
\vspace{-2mm}
\paragraph{Task Formulation:}
Formally, the task can be represented as a partially observable Markov decision process, described by the tuple $(U, S, O, A, T, R)$ where $U$ represents the space of natural language task instructions, $S$ is the set of environment's states and $O$ is the corresponding observation space providing textual information on the environment's state, and 
$A$ denotes the discrete action space defined in natural language. The transition dynamics $T: S \times A \rightarrow S$ govern how the environment's state evolves after an action, and the reward $R$ provides scalar feedback indicating the overall task performance of the agent.
\vspace{-4mm}
\paragraph{Task Execution:} At the start of each episode, the agent is provided with a task instruction $u \in U$ and an initial observation $o_0$ that describes the environment’s initial state $s_0$. 
At any discrete time step $t$, the agent takes an action according to its learned policy $\pi_{\theta}$:
\[
a_t \sim \pi_{\theta}(\cdot \mid \tau_{t-1}).
\]
Here $\tau_{t-1} = (u, s_0, a_0, \ldots, a_{t-2}, s_{t-1})$ is the episode's history till the previous time-step. Upon executing $a_t$, the environment transitions to a new state $s_t = T(s_{t-1}, a_t)$ and produces the corresponding observation $o_t$ for the LLM.  This iterative process continues until the task is successfully completed or the maximum number of steps is reached. At the end of episode, the environment provides a scalar reward summarizing the overall task performance of the agent.

\begin{tcolorbox}[
    colback=gray!5,
    colframe=black!60,
    title=\textbf{Example Task for a Household LLM Agent},
    fonttitle=\bfseries,
    boxrule=0.6pt,
    arc=3pt,
    left=4pt, right=4pt, top=4pt, bottom=4pt,
    before skip=4pt,
    after skip=4pt
]
\small

\textbf{Task:} 
$u =$ \textit{``Clean a tomato and put it on the shelf next to the stove.''}

\medskip

\textbf{Initial Observation:}  
$o_0 =$ \textit{``You are in the middle of a kitchen. You observe a refrigerator, washbasin, microwave, shelf1 next to stove, and shelf2 next to basin.''}

\medskip

To complete the task, the agent samples a sequence of sub-tasks from its learned policy, such as opening the refrigerator to find a tomato, washing it in the washbasin, and placing it on shelf1 next to the stove. After each action, the environment returns a new textual observation. 

\medskip

For example, after action $a_t =$ \textit{open[refrigerator]}, the environment might return:
$o_{t+1} =$ \textit{``The refrigerator is now open. You see milk, tomato, cheese, and carrots.''} \\

This iterative action–observation loop continues until the task is completed or the step limit is reached, with the correctness of each action contributing to the overall success.
At the end of the episode, the environment provides a scalar reward summarizing the agent’s task performance, e.g., a value of $1$ if the tomato is washed and correctly placed.

\end{tcolorbox}

\subsection{Supervised Fine-Tuning of LLM Agents}
\label{sec:sft}
To equip a large language model (LLM) with core agentic capabilities, we perform \textit{supervised fine-tuning} (SFT) on expert demonstrations, aligning the model’s policy with trajectories that exemplify correct reasoning, decision-making, and action execution in interactive environments~\citep{ipr}.
\vspace{-3mm}
\paragraph{Expert Demonstrations and the ReAct Paradigm:}
In agentic settings, the model must not only produce the correct action but also reason about why that action is appropriate given the current context. To capture this reasoning–action interplay, we adopt the \textit{ReAct} (Reasoning and Acting) format~\citep{yao2022react} of the expert demonstrations for SFT. Here, each step taken by the agent consists of a natural language reasoning trace followed by an executable action.  An example on how the ReAct step looks like is as follows: \texttt{Thought: I need to open the refrigerator to check for tomato.
Action: open[refrigerator].}

Such paired data explicitly teach the model how to alternate between reflective reasoning (\texttt{Thought}) and executable (\texttt{Action}) in the environment, promoting explainability in its downstream behavior.

\vspace{-0.5mm}
\paragraph{Expert Trajectory Dataset Construction:}
Let $ET = \{(u^{(i)}, \tau^{(i)})\}_{i=1}^{|ET|}$ denote a collection of \textit{E}xpert \textit{T}rajectories, where $u^{(i)}$ is the natural language task instruction and $\tau^{(i)} = (s_o, a_0, s_1, \ldots, a_{final}, s_{final})$ is the sequence of actions and corresponding states (expressed as observations in natural language) from the expert’s interaction with the environment.
Each trajectory is annotated in ReAct form, providing the reasoning text before each action. 
These demonstrations can originate from human experts, high-performing teacher models such as GPT~\citep{gpt}, or curated datasets collected through scripted interaction with simulation environments.

\paragraph{Training Objective:}
During SFT, the LLM learns to imitate the expert’s behavior by maximizing the likelihood of the expert trajectory given the task instruction:
\[
\mathcal{L}_{\mathrm{SFT}}(\theta)
= - \mathbb{E}_{(u, \tau) \sim ET}
  \big[ \log \pi_{\theta}(\tau \mid u) \big],
\]
Here, $\theta$ is the set of model parameters. In practice, the joint probability of the entire trajectory can be decomposed into a sequence of conditional action probabilities:
\[
\pi_{\theta}(\tau \mid u) 
= \prod_{t=1}^{final} \pi_{\theta}(a_t \mid u, s_o, a_0, \ldots, s_{t-1}),
\]
which leads to a token-level autoregressive training objective:
\[
\mathcal{L}_{\mathrm{SFT}}(\theta)
= - \mathbb{E}_{(u, \tau) \sim ET}
  \Big[ \sum_{t=1}^{final} 
  \log \pi_{\theta}(a_t \mid u, \tau_{t-1}) \Big],
\]
where $\tau_{t-1}$ denotes the trajectory history up to step $t-1$.
This objective encourages the agent to reproduce the expert’s next action at each step, conditioned on the task and its full prior context on its previous steps and the environment’s intermediate observations.

\subsection{Conformal Prediction}
Conformal prediction~\citep{cp} is a statistical framework for quantifying how well a new input sample aligns with a reference data distribution. At its core, the framework relies on a \emph{non-conformity measure} (NCM), a real-valued function that quantifies the extent to which an input deviates from the behavior observed in the reference data. 
Given a dataset $X = \{x_1, x_2, \ldots, x_l\}$ drawn i.i.d.\ from an underlying data distribution $\mathcal{D}$ of interest, a non-conformity score $\alpha_x$ is assigned to the input $x$ by the NCM defined on $X \cup \{x\}$. Larger value of $\alpha_x$ indicates greater deviation from $\mathcal{D}$ and, consequently, a higher likelihood that the sample is atypical.
\vspace{-4mm}
\paragraph{Classical Conformal Anomaly Detection:}
Conformal Anomaly Detection (CAD) utilizes the non-conformity score to assess the likelihood that an unseen input belonging to the same distribution as $\mathcal{D}$. This is done by computing $p$-value of the input $x$ by comparing its non-conformity score $\alpha_x$ with those of the datapoints in $X$ from the NCM defined on the new set $X \cup \{x\}$:
\begin{equation}
p\text{-}value = \frac{|\{\, i \in \{1, \ldots, l\} : \alpha_x \leq \alpha_i \,\}| + 1}{l + 1}.
\label{eq:cad}
\end{equation}
If $x$ is sampled from $\mathcal{D}$, its $p$-value will tend to be large; conversely, inputs exhibiting high deviation from $\mathcal{D}$ yield smaller $p$-values. An input is deemed anomalous when $p\text{-}value < \epsilon$, where $\epsilon \in (0,1)$ is a user-defined significance level controlling the allowable false-alarm probability.
\vspace{-3mm}
\paragraph{Inductive Conformal Anomaly Detection~\citep{icp}:}
While the classical formulation in~\eqref{eq:cad} is statistically sound, recomputing NCM scores across the entire set $X$ for every test input is computationally expensive. 
The \emph{Inductive Conformal Anomaly Detection} (ICAD) framework mitigates this cost by partitioning the data into a proper training subset $X_{\mathrm{tr}} = \{x_1, \ldots, x_m\}$ and a calibration subset $X_{\mathrm{cal}} = \{x_{m+1}, \ldots, x_l\}$. 
The NCM is defined on $X_{\mathrm{tr}}$ and then evaluated for each calibration datapoint to produce the set of calibration scores $\{\alpha_j\}_{j=m+1}^l$. 
For a new input $x$, its non-conformity is assessed by comparing its score $\alpha_x$ against the calibration scores:
\begin{equation}
p\text{-}value = \frac{|\{\, j \in \{m+1, \ldots, l\} : \alpha_x \leq \alpha_j \,\}| + 1}{l - m + 1},
\label{eq:icad}
\end{equation}
and, again, detecting $x$ as anomalous if the computed $p\text{-}value < \epsilon$. This inductive formulation permits all calibration scores to be pre-computed offline, thereby enabling efficient inference while retaining the theoretical validity of the conformal framework.
\vspace{-4mm}
\paragraph{Statistical Validity and Error Control~\citep{icp}:}
Under the standard i.i.d.\ assumption that both calibration and test samples are drawn from the same distribution $\mathcal{D}$, the $p$-values obtained via~\eqref{eq:icad} are uniformly distributed in the interval $(0,1)$. 
Consequently, the probability of a false alarm—incorrectly identifying an in-distribution sample as anomalous—is provably bounded by the significance threshold $\epsilon$:
\begin{equation}
\Pr\!\left[p(x \in \mathcal{D}) < \epsilon\right] \leq \epsilon.
\label{eq:icad_guarantees}
\end{equation}

The efficacy of the conformal framework depends on the underlying NCM. A variety of NCMs have been proposed in literature, employing methods such as $k$-nearest neighbors~\citep{icp}, variational autoencoders~\citep{vanderbilt}, memory prototypes~\citep{yang2024memory}, and transformation equivariance~\citep{idecode, kaur2024tcps}. We propose using the step-wise reward to define the NCM for labeling each step as sampled from the distribution of successful or failing step.

\section{Identifying Directions for Temporal Concepts in LLM Agents}
\label{sec:tech}
For an LLM agent trained to perform a multi-step sequential task, we hypothesize that the directions of success and failure for the task become linearly separable in the agent's internal representation space across layers and time. We aim to validate this hypothesis by training linear probes to classify the agent's activation space as success or failure at each step of the task.

The notion of success (or failure) is quantified by the \textit{step-wise reward} assigned to the agent at each sequential step taken to accomplish the task. Here, we provide details on a) generating these step-wise rewards via Monte-Carlo sampling of agent's trajectory from its learned policy on performing these tasks, b) leveraging inductive conformal anomaly detection (ICAD) framework for labeling these step-wise rewards as success (or failure) with bounded probability on making labeling errors, and c) training linear probes to classify the agent's internal activation space at each step as success or failure from the labeled step-wise reward.

\subsection{Generating Step-Wise Rewards}
Traditional learning based approaches rely solely on the \textit{final reward} $r$ assigned to the agent's complete trajectory on its ability to complete the task. This obscures the contribution of individual steps towards the goal. To achieve granular interpretability, we propose to use \textit{step-wise rewards} $\{r_t\}_{t=1}^T$, where $r_t$ represents the quantitative measure of success for the agent’s partial trajectory $\tau_t$ till time $t$:
$$\tau_t=(s_0, a_0, s_1, a_1, \dots, s_{t-1}, a_{t-1}, s_t).$$

To estimate $r_t$, we follow~\citet{ipr}'s approach on
performing Monte Carlo sampling on the agent's action space from its learned policy $\pi_{\theta}$ conditioned on the observed trajectory. Specifically, given $\tau_t$, we generate $N$ complete, subsequent expected trajectories $$e^{(i)}= (a^i_{t+1}, s^i_{t+1}, \ldots a^i_{final}, s^i_{final}),$$ by performing iterative Monte Carlo sampling on the agents action space starting from $a_t$ till $a^i_{final}$.

The step-level reward $r_t$ at time $t$ is calculated as:
$$
r_t =
\begin{cases}
\frac{1}{N} \sum_{i=1}^{N} r(\tau_t, e^{(i)}) & \text{for } t < n, \\
r(\tau_t), & \text{for } t = n.
\end{cases}
$$

This procedure allows the model to evaluate the \textit{expected future success probability} measured in terms of step-wise rewards that are conditioned on its current state. The use of Monte Carlo sampling transforms interpretability into a probabilistic, forward-looking signal, better aligning with how agents internally (intends to) plan the task execution.

\subsection{Labeling Step-Wise Rewards}
Once step-wise rewards are estimated, the key challenge is determining \textit{when a reward should be interpreted as success or failure.} We introduce the \textbf{Conformal Prediction–based labeling mechanism} to statistically calibrate both success and failure labels with bounded guarantees on making the labeling error.

Given the calibration set of step-wise rewards for both successful and failure steps, we can calculate the (non-)conformity of $r_t$ w.r.t success as well as failure. This gives us two $p$-values for the step: one corresponding to the likelihood of it being a successful step ($p_s$), and the other corresponding to the likelihood of it being a failure step ($p_f$). With the intuition of higher rewards for successful steps than failure, we propose $$\alpha_s = 1-r_t, \text{ and } \alpha_f = r_t.$$ as the non-conformity score with respect to the successful and failure steps respectively. $r_t$ is then labeled as:
\begin{equation}
\text{label}(r_t) =
\begin{cases}
\text{success},  & \text{ if } p_s \ge \epsilon_s \text{ and }  p_f < \epsilon_f, \\
\text{failure}, & \text{ if } p_f \ge \epsilon_f \text{ and }   p_s < \epsilon_s.
\end{cases}
\label{eq:label_reward}
\end{equation}

\begin{theorem}[Bounded Guarantees on making Labeling Errors]
The probability of labeling a successful step-wise reward as failure (or False Negative Rate) is strictly bounded by $\epsilon_s$, and the probability of labeling a failure step-wise reward as success (or False Positive Rate) is strictly bounded by $\epsilon_f$.  
\end{theorem}

\begin{proof}
The proof without strict bounds, i.e. without the `and' (intersection) conditions in~\eqref{eq:label_reward}:
\begin{equation*}
\text{label}(r_t) =
\begin{cases}
\text{success},  & \text{ if } p_s \ge \epsilon_s, \\
\text{failure}, & \text{ if } p_f \ge \epsilon_f. 
\end{cases}
\end{equation*}
follows directly from the statistical guarantees of the inductive conformal anomaly detection (ICAD) framework~\eqref{eq:icad_guarantees}. 
In other words, the probability of $p$-value for successful $r_t$ less than $\epsilon_s$  is bounded by $\epsilon_s$. The second condition $p_f < \epsilon_f$ is an intersection, which reduces the overall probability. Therefore, the false negative rate is strictly bounded by $\epsilon_s$. Similarly, the probability of $p$-value for failure $r_t$ less than $\epsilon_f$ is bounded by $\epsilon_f$. The second condition $p_s < \epsilon_s$ is an intersection, which reduces the overall probability.  Therefore, the false positive rate is strictly bounded by $\epsilon_f$.

\end{proof}

\subsection{Representation Probing Across Timesteps}
At each time step $t$, the model’s hidden representation captures contextualized knowledge of its decisions till time $t$. We define the internal state $\mathbf{h}_t^L$ of an LLM agent as its residual stream activations at a specific layer $L$ and at the last token position corresponding to the trajectory $\tau_t$ till time $t$. This state $\mathbf{h}_t^L$ is the object of our interpretability study.

The challenge is to map these latent representations to an interpretable success/failure signal at every timestep. Having obtained calibrated step-wise reward labels, we investigate whether the corresponding model’s hidden representations for success and failure are linearly separable.
For this, we then train \textbf{linear probes} $P^L_{t}$ to classify  $h_t^L$  as success vs.\ failure: $\hat{y}_{t}^L = \sigma(W_{t}^L h_t^L + b^L_{t}),$ where $W_{t}^L$ and $b_{t}^L$ are the (weight and bias) probe parameters.

Classification accuracy of probe provides a quantitative measure of how distinctly the model encodes success and failure trajectories within its internal representation space~\citep{alain2017understanding, geom_of_truth}. 


\section{Experimental Results}
\label{sec:exp}
\subsection{Case Study I: ScienceWorld}

Our first case study is on sequential tasks in ScienceWorld~\citep{sciworld}, a large-scale text-based environment for evaluating an LLM agent’s ability to perform scientific reasoning and procedural tasks. 
\vspace{-4mm}
\paragraph{Environment:} 
Each instance of the environment represents a small virtual world inspired by elementary science domains—such as physics, chemistry, and biology—where the agent must explore, manipulate, and reason about objects to accomplish experiment-style goals. The environment comprises interconnected rooms (e.g., greenhouse, laboratory, workshop), each populated with interactive objects that support diverse affordances. Agents operate entirely through natural-language commands and receive textual feedback describing environmental state changes. Fig.~\ref{fig:sciworld_ex} in the Appendix illustrates a representative task, “Testing Conductivity”, along with the corresponding action-observation trajectory. Tasks in ScienceWorld are typically long-horizon and a normalized reward in $[0, 1]$, which is assigned at the end of each episode to reflect the agent’s overall task performance. Specifically, each task is decomposed into multiple sub-goals, and the final reward is computed based on how many of these sub-goals the agent achieves, thereby enabling a fine-grained evaluation of procedural and scientific reasoning capabilities. Examples of these sub-goals are illustrated in Fig.~\ref{fig:sw_steering_results}, where the preferred sequence of steps is mentioned in the task description.
\vspace{-4mm}
\paragraph{LLM Agent:} 
We train the \texttt{Llama-2-7B} model~\citep{llama-2} on $60\%$  ($889$ out of $1443$) of the training trajectories. The agent is trained using Supervised Fine-Tuning (SFT) on a curated dataset of expert trajectories formatted in the ReAct paradigm~\citep{yao2022react}, which explicitly interleaves natural language \texttt{Thought} and executable \texttt{Action} steps. The remaining $40\%$ is split equally between the calibration set for conformal labeling of step-wise rewards and the training set of probes on residual stream activations of all $32$ layers on the last token of the entire trajectory till time $t$. We observe that the trained agent is mostly able to successfully complete the assigned task within $10$ timesteps. 
\vspace{-4mm}
\paragraph{Testing Scenarios:} The environment offers two types of test scenarios: seen and unseen. Seen test set comprises of those tasks (or variations of tasks) that the agent encountered and learned from during its training phase. We refer to this test set as \textit{in-distribution} because it falls within the scope of the data the model was exposed to. Unseen test set comprises of those tasks that the agent never encountered during its training. These involve novel combinations of objects, new environmental layouts, or even entirely new scientific concepts. We refer to this test set as \textit{out-of-distribution} (OOD). We test on the entire set of $360$ in-distribution and $165$ OOD tasks.
\vspace{-4mm}
\paragraph{Conformal Thresholds:} We set $\epsilon_s = \epsilon_f = 0.1$ for labeling step-wise rewards as success or failure all timesteps. This strictly bounds both the false negative and false positive labeling errors to $10\%$~\citep{idecode}.

\vspace{-4mm}
\paragraph{Results and Observations:} Tables~\ref{tab:accuracy_train} and~\ref{tab:accuracy_test_unseen} show accuracy on iD  and OOD test set respectively from timesteps $t=2 \text{ to } 10$.  At timestep $1$, the model is given instructions on its role in the ScienceWorld environment and it always (irrespective of success or failure) responds with an `OK'.
Probes achieve significantly high accuracy with upto $100\%$ in most test cases for iD set and good accuracy in most test cases for OOD set except for one test case ($50\%$ at $t=10$ for layer $8$). 

Tables~\ref{tab:f1_train} and~\ref{tab:f1_test_unseen} show F1 scores on the in-distribution and OOD test sets, respectively. Similar to results on accuracy, F1 scores are also high on most test cases of the iD set except for one test case ($0.67$ at $t=10$ for layer $8$). These scores are also high on most test cases of the OOD set except for three cases: $0.67$ at $t=3$ for layers $24$ and $32$ and for layer $8$ at $t=8$. We observe similar results across all layers of the model, and selected early (layer $8$), middle ($16$ and $24$) and later ($32$) layers to show the results. These results validate our hypothesis that directions of success and failure become separable in an LLM fine-tuned to perform sequential tasks. 

\begin{table}[t]
  \centering
  \resizebox{1\linewidth}{!}{
  \begin{tabular}{cccccccccc}
    \toprule
    Layer & $t{=}2$ & $t{=}3$ & $t{=}4$ & $t{=}5$ & $t{=}6$ & $t{=}7$ & $t{=}8$ & $t{=}9$ & $t{=}10$ \\
    \midrule
     8  & 100 & 100 & 100 & 100 &  100 & 83 &  91 &  91 &  100 \\
    16  & 100 & 100 & 100 & 100 & 100 & 92 & 100 & 91 & 100  \\
    24  & 100 & 93 & 100 & 100 & 100 & 92 & 100 & 91 & 100 \\
    32  & 100 & 93 & 100 & 100 & 100 & 92 & 91 & 91 & 100 \\
    \bottomrule
  \end{tabular}
  }
  \caption{Accuracy($\%$) of Linear Probes on in-distribution Test Set of ScienceWorld across Layers and Timesteps.}
  \label{tab:accuracy_train}
\end{table}

\begin{table}[t]
  \centering
  \resizebox{1\linewidth}{!}{
  \begin{tabular}{cccccccccc}
    \toprule
    Layer & $t{=}2$ & $t{=}3$ & $t{=}4$ & $t{=}5$ & $t{=}6$ & $t{=}7$ & $t{=}8$ & $t{=}9$ & $t{=}10$ \\
    \midrule
     8  & 100 & 92 & 100 & 94 & 92 & 80 & 75 & 75 & 50  \\
    16  & 100 & 92 & 93 & 94 & 92 & 80 & 75 & 75 & 100 \\
    24  & 100 & 92 & 100 & 94 & 92 & 73 & 75 & 75 & 100  \\
    32  & 100 & 92 & 100 & 94 & 92 & 73 & 75 & 75 & 100 \\
    \bottomrule
  \end{tabular}
  }
  \caption{Accuracy($\%$) of Linear Probes on OOD Test Set of ScienceWorld across Layers and Timesteps.}
  \label{tab:accuracy_test_unseen}
\end{table}

\begin{table}[t]
  \centering
  \resizebox{1\linewidth}{!}{
  \begin{tabular}{cccccccccc}
    \toprule
    Layer & $t{=}2$ & $t{=}3$ & $t{=}4$ & $t{=}5$ & $t{=}6$ & $t{=}7$ & $t{=}8$ & $t{=}9$ & $t{=}10$ \\
    \midrule
     8  & 1.00 & 0.95 & 1.00 & 0.97 & 0.96 & 0.88 & 0.86 & 0.80 & 0.67 \\
    16  & 1.00 & 0.95 & 0.96 & 0.97 & 0.96 & 0.88 & 0.86 & 0.80 & 1.00 \\
    24  & 1.00 & 0.95 & 1.00 & 0.97 & 0.96 & 0.83 & 0.86 & 0.80 & 1.00 \\
    32  & 1.00 & 0.95 & 1.00 & 0.97 & 0.96 & 0.83 & 0.86 & 0.80 & 1.00 \\
    \bottomrule
  \end{tabular}
  }
  \caption{F1 Score of Linear Probes on in-distribution Test Set of ScienceWorld across Layers and Timesteps.}
  \label{tab:f1_train}
\end{table}

\begin{table}[t]
  \centering
  \resizebox{1\linewidth}{!}{
  \begin{tabular}{cccccccccc}
    \toprule
    Layer & $t{=}2$ & $t{=}3$ & $t{=}4$ & $t{=}5$ & $t{=}6$ & $t{=}7$ & $t{=}8$ & $t{=}9$ & $t{=}10$ \\
    \midrule
     8  & 1.00 & 1.00 & 1.00 & 1.00 & 1.00 & 1.00 & 0.67 & 0.86 & 0.80 \\
    16  & 1.00 & 1.00 & 1.00 & 1.00 & 1.00 & 0.80 & 1.00 & 0.80 & 1.00 \\
    24  & 1.00 & 0.67 & 1.00 & 1.00 & 1.00 & 0.80 & 1.00 & 0.80 & 1.00 \\
    32  & 1.00 & 0.67 & 1.00 & 1.00 & 1.00 & 0.80 & 0.86 & 0.80 & 1.00 \\
    \bottomrule
  \end{tabular}
  }
  \caption{F1 Score of Linear Probes on OOD Test Set across of ScienceWorld Layers and Timesteps.}
  \label{tab:f1_test_unseen}
\end{table}

\subsection{Case Study II: AlfWorld}

The second case study evaluates our conformal interpretability framework on the \emph{ALFWorld} environment \citep{alfworld}. AlfWorld serves as a widely adopted benchmark for autonomous LLM agents in interactive embodied settings, requiring complex sequential reasoning and navigation in simulated household environment~\citep{embodied_ai_survey, agentic_rl_survey}.
\vspace{-4mm}
\paragraph{Environment:}
ALFWorld is a text-based environment that grounds language instructions in a simulated household world derived from ALFRED~\citep{alfred}. Tasks involve multi-step goals such as fetching, cleaning, heating, or putting away objects (e.g., “put a clean mug in the cabinet”), requiring sequential decision-making, long-horizon planning, and interaction with diverse objects and room states. The agent receives textual observations and issues text-based actions (e.g., \texttt{go to kitchen}, \texttt{pick up mug}); performance is evaluated by whether the final goal is successfully completed. Fig.~\ref{fig:alfworld_trajectory} (Appendix) shows an example trajectory for a task of cleaning a tomato and placing it on a shelf next to the stove.


\vspace{-4mm}
\paragraph{LLM Agent and Training:}
Similar to ScienceWorld, here also we train Llama-2-7B~\citep{llama-2} on $60\%$ ($1710$ out of $2851$) of the training trajectories and split the remaining equally between calibration set and training the probes on all layers and timesteps.

\vspace{-4mm}
\paragraph{Conformal Thresholds:} Again, we set $\epsilon_s = \epsilon_f = 0.1$ for all timesteps $t \in \{1, 2, \ldots, 10\}$, strcitly bounding both false negative and false positive labeling errors to $10\%$.

\vspace{-4mm}
\paragraph{Results and Discussion:}
Tables~\ref{tab:alfworld_accuracy_test_unseen} and ~\ref{tab:alfworld_f1_test_unseen} show the accuracy and F1 scores of the trained probes on the test set of the AlfWorld, respectively. The accuracy varies from $60\%$ (layer $8$ at $t=6$ and layer $24$ at $t=4$) to $95\%$ (at $t=2$), and the F1 score varies from $0.56$ (layer $24$ at $t=4$ to $0.95$ (at $t=2$). Although these results are comparable to the probe accuracies reported for distinguishing truth and falsehood directions in factual question settings~\citep{geom_of_truth}, they are lower than our results on ScienceWorld. We hypothesize that this difference arises because ALFWorld provides only a single final reward upon task completion, whereas ScienceWorld offers intermediate rewards for sub-goals at multiple steps. This denser reward structure in ScienceWorld yields more precise step-wise feedback, leading to improved predictive performance. 

\begin{table}[htbp]
  \centering
  
  
  


  \resizebox{1\linewidth}{!}{
  \begin{tabular}{cccccccccc}
    \toprule
    Layer & $t{=}2$ & $t{=}3$ & $t{=}4$ & $t{=}5$ & $t{=}6$ & $t{=}7$ & $t{=}8$ & $t{=}9$ & $t{=}10$ \\
    \midrule
     8  & 95 & 85 & 75 & 75 & 60 & 80 & 75 & 70 & 75 \\
    16  & 95 & 90 & 65 & 65 & 75 & 65 & 80 & 80 & 75 \\
    24  & 95 & 90 & 60 & 85 & 80 & 60 & 80 & 75 & 75 \\
    32  & 95 & 85 & 75 & 85 & 80 & 75 & 90 & 80 & 75 \\
    \bottomrule
  \end{tabular}
  }
  \caption{Accuracy ($\%$) of Linear Probes on ALFWorld Test Set across Layers and Timesteps.}
  \label{tab:alfworld_accuracy_test_unseen}

  \vspace{0.5em} 

  \resizebox{1\linewidth}{!}{
  \begin{tabular}{cccccccccc}
    \toprule
    Layer & $t{=}2$ & $t{=}3$ & $t{=}4$ & $t{=}5$ & $t{=}6$ & $t{=}7$ & $t{=}8$ & $t{=}9$ & $t{=}10$ \\
    \midrule
     8  & 0.95 & 0.82 & 0.71 & 0.76 & 0.64 & 0.82 & 0.76 & 0.62 & 0.71 \\
    16  & 0.95 & 0.89 & 0.59 & 0.63 & 0.74 & 0.63 & 0.78 & 0.78 & 0.67 \\
    24  & 0.95 & 0.89 & 0.56 & 0.84 & 0.80 & 0.60 & 0.80 & 0.71 & 0.67 \\
    32  & 0.95 & 0.82 & 0.74 & 0.86 & 0.80 & 0.74 & 0.89 & 0.78 & 0.67 \\
    \bottomrule
  \end{tabular}
  }
  \caption{F1 Score of Linear Probes on ALFWorld Test Set across Layers and Timesteps.}
  \label{tab:alfworld_f1_test_unseen}
\end{table}

\textbf{Steerability of Models Towards Success:} Using the labeled internal representations provided by our framework, we apply the existing steering approach Representation Engineering (RepE)~\citep{repe} to compute contrastive activation directions between successful and failing trajectories, and add these directions during inference to steer the agent’s internal states toward success while preserving overall performance. Motivated by the empirical observation that success and failure trajectories begin to diverge at timestep $3$, we perform a single early intervention at step $3$ on a supervised fine-tuned Llama-2-7B agent in ScienceWorld. We use a small steering coefficient of 0.025 as selected by following RepE’s procedure. This intervention yields a 1.1\% accuracy boost, which is competitive with substantially more expensive methods—Best-of-$N$ sampling, Rejection Sampling Fine-Tuning (RFT), and Direct Preference Optimization (DPO)—that obtain 2.8\%, 4.2\%, and 6.8\% accuracy gains, respectively, on the same test settings.

Fig.~\ref{fig:sw_steering_results} illustrates two test cases where the steered model corrects common failure modes: (a) failing to follow the correct sequence of sub-goals and (b) drifting off track mid-trajectory. The reported steering results are preliminary as we applied an off-the-shelf steering approach (RepE) only at timestep $3$. RepE has been proposed to steer standalone inputs. Leveraging the proposed framework for developing and applying steering approaches for temporal data across steps is one of the future directions. 


\begin{figure}[!t]
\centering
\setlength{\abovecaptionskip}{2pt}

\begin{subfigure}{\linewidth}
    \centering
    \includegraphics[width=0.9\linewidth]{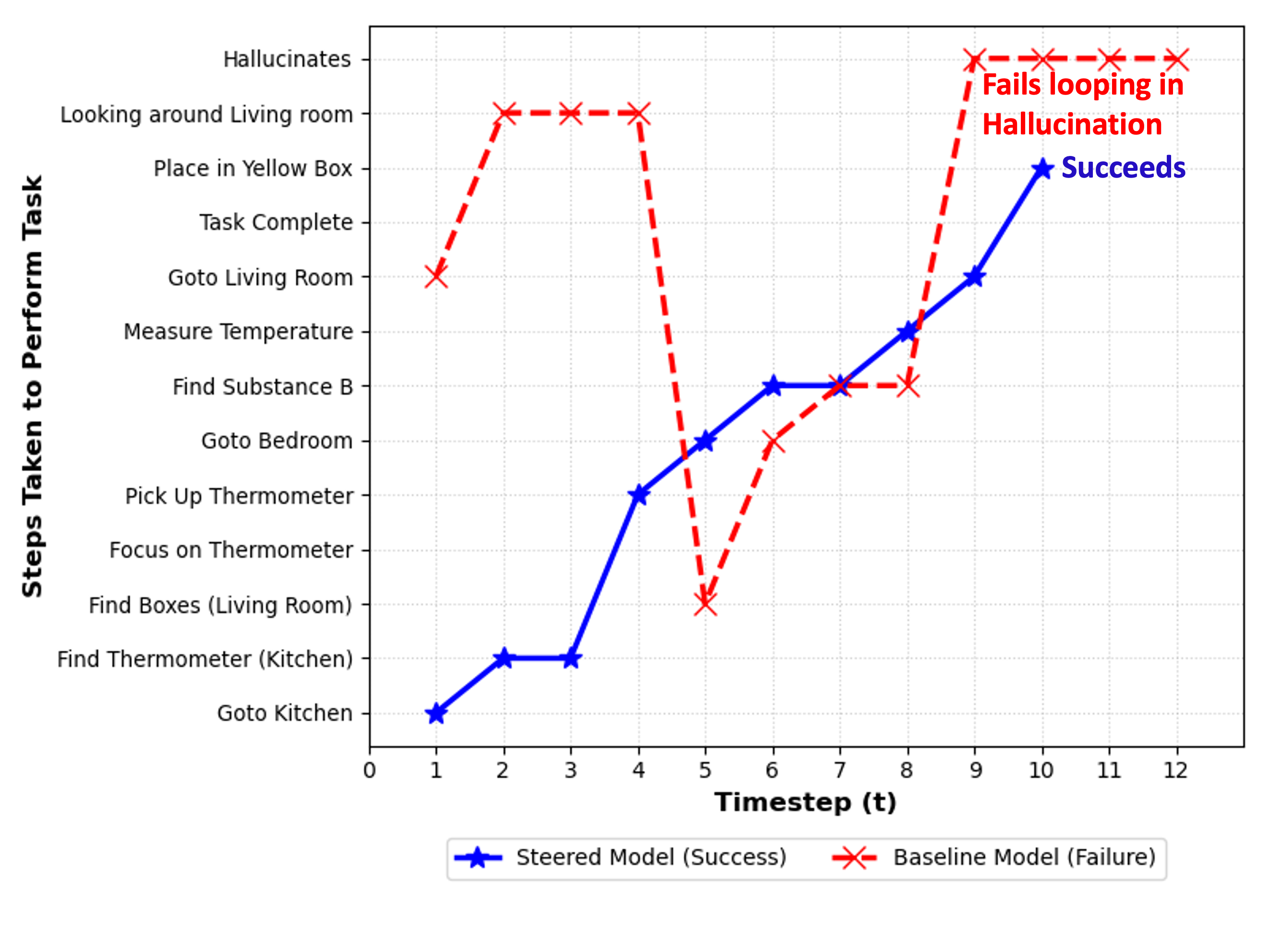}
    \caption{\footnotesize{
    Task: ``Measure the temperature of substance B in bedroom. First, focus on the thermometer. Next, focus on the unknown substance B. Then, if temperature $> 0.0^\circ$C $\rightarrow$ place it in the yellow box; if $< 0.0^\circ$C $\rightarrow$ place it in the purple box. Boxes are in the living room.'' The steered model (blue) prioritizes the correct instruction order and completes the task, while the baseline (red) follows the wrong order and loops due to hallucinating possession of the thermometer. 
    }}
    \label{fig:steered_task2}
\end{subfigure}

\vspace{4pt}

\begin{subfigure}{\linewidth}
    \centering
    \includegraphics[width=0.9\linewidth]{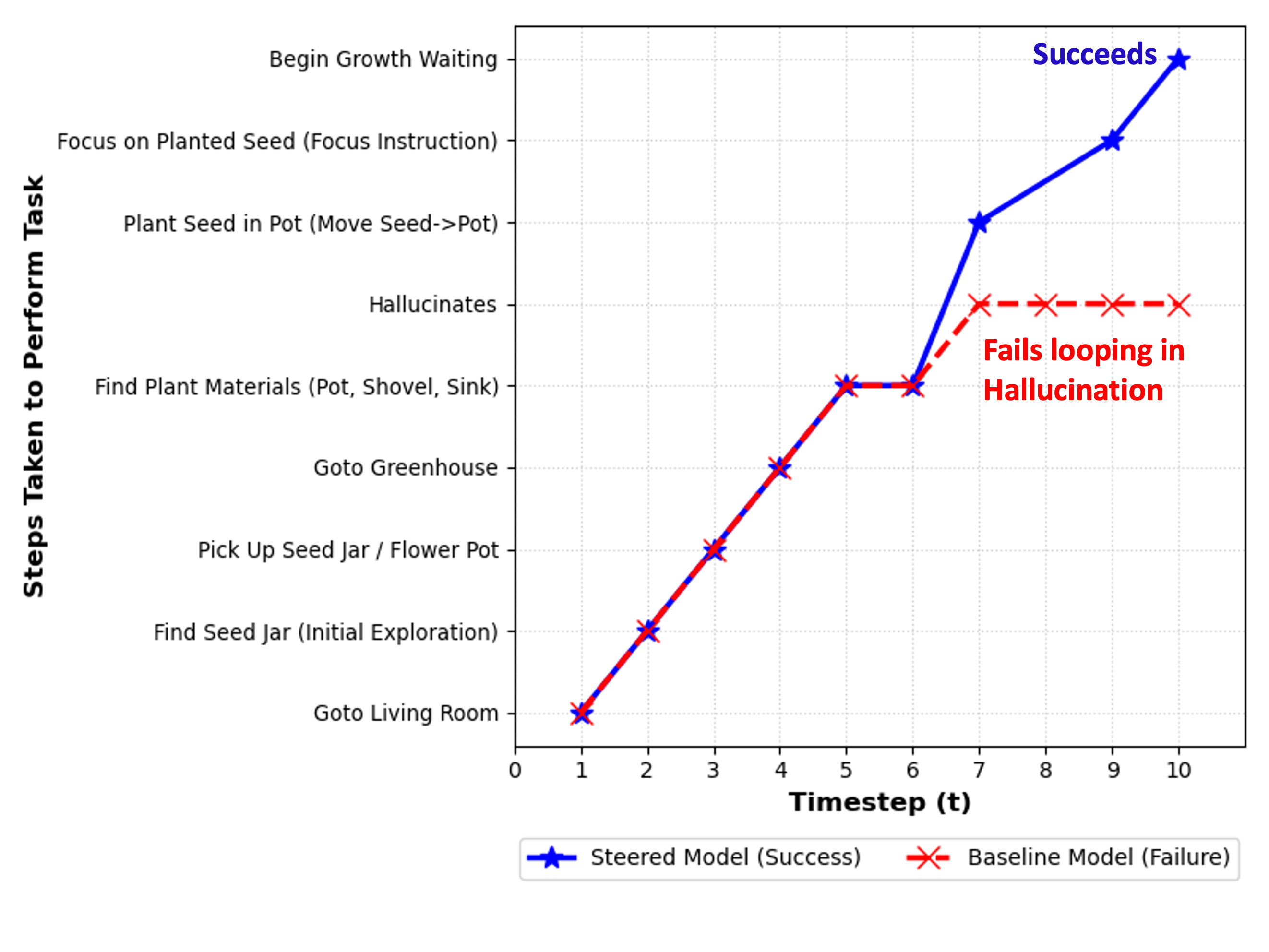}
    \caption{\footnotesize{
    Task: ``Grow Plant to Reproduction. Seeds can be found in the living room. First, focus on a seed. Then, modify the environment to grow the plant until it reaches the reproduction stage.'' The baseline model (red) initially follows the correct sequence but later drifts due to hallucination. 
    }}
    \label{fig:steered_task3}
\end{subfigure}
\caption{\footnotesize{Comparison of baseline (SFT Llama-2-7B) and steered LLM agents -- Steering along learned \emph{success directions} mitigates reasoning drift, reduces hallucinations, and improves task completion.}}
\label{fig:sw_steering_results}
\end{figure}
\section{Conclusion}
\label{sec:conc}
This work introduces a conformal time-series interpretability framework for analyzing the evolving internal representations of LLM agents on sequential tasks. By combining step-wise reward modeling with inductive conformal prediction, it assigns statistically calibrated success/failure labels to internal states at each timestep, enabling fine-grained temporal analysis and linear probing of success vs.\ failure directions. Experiments on two complex interactive environments show that these directions are linearly separable and can be used to steer agents toward more successful outcomes, supporting early detection of misalignment and step-level intervention. Future work will extend the framework to multimodal embodied settings and explore proactive steering of autonomous LLM agents via interpretable temporal feedback.


\section*{Acknowledgments}
This material is based on work supported by the United States Air Force and Defense Advanced Research Projects Agency (DARPA) under Contract No. FA8750-23-C-0519 and under Agreement No. HR0011-24-9-0424. The views, opinions and/or findings expressed are those of the author and should not be interpreted as representing the official views or policies of the Department of War or the U.S. Government. The authors also acknowledge the use of Advanced Research Computing Technology Innovation and Collaboration (ARCTIC) resources at Georgia State University, supported by the National Science Foundation under Major Research Instrumentation (MRI) grant CNS-1920024 and the Research Infrastructure: CC Compute-Campus award for Pioneering Research-Oriented Flexible Cyberinfrastructure (PROFLEX-CI), grant 2430193.

\newpage
{
    \small
    \bibliographystyle{ieeenat_fullname}
    \bibliography{main}
}

\newpage
\newpage
\section{Supplementary Material}
\label{appendix}
\subsection{Task Examples from ScienceWorld and AlfWorld}

Figure~\ref{fig:sciworld_ex} illustrates a representative task, “Testing Conductivity” in ScienceWorld, along with the corresponding action-observation trajectory, and Figure~\ref{fig:alfworld_trajectory} shows an example trajectory of the agent-environment interaction on one of the household's task of ``cleaning a tomato and putting it on the shelf next to stove''.

\begin{figure*}[!t]
    \centering
    \caption{Task examples from ScienceWorld and ALFWorld. \textbf{(Top)} An example of a successful trajectory for the "Test Conductivity" task in \texttt{ScienceWorld}, demonstrating the sequential reasoning required by the agent to accomplish the task. \textbf{(Bottom)} An example of a successful trajectory for the ``Cleaning a tomato and putting it on the shelf next to stove'' task in the ALFWorld environment, demonstrating the multi-step nature of a typical sequential task in the environment requiring commonsense, navigation, and object manipulation.}
    \label{fig:task_examples}

    \vspace{4pt}
    {\small\textbf{ScienceWorld: Test Conductivity}}
    \vspace{2pt}

    \begin{tabular}{c|l|l}
        \toprule
        \textbf{Step} & \textbf{Agent's Action} & \textbf{Environment Observation (Simplified)} \\
        \midrule
        \textbf{1} & \texttt{go to tool room} & You are in the tool room. You see a light bulb and a battery. \\
        \textbf{2} & \texttt{pick up light bulb} & You picked up the light bulb. \\
        \textbf{3} & \texttt{pick up battery} & You picked up the battery. \\
        \textbf{4} & \texttt{go to workshop} & You are in the workshop. You see a metal fork and a plastic cup. \\
        \textbf{5} & \texttt{pick up metal fork} & You picked up the metal fork. \\
        \textbf{6} & \texttt{use battery on light bulb} & The battery and the light bulb are now connected in a circuit. \\
        \textbf{7} & \texttt{use metal fork on light bulb} & The light bulb illuminates! The metal fork is electrically conductive. \\
        \bottomrule
    \end{tabular}
    \label{fig:sciworld_ex}

    \vspace{10pt}
    {\small\textbf{ALFWorld: Clean tomato}}
    \vspace{2pt}

    \begin{tabular}{c|l|l}
        \hline
        \textbf{Step} & \textbf{Agent's Action} & \textbf{Environment Observation (Simplified)} \\
        \hline
        1 & \texttt{open refrigerator} & The refrigerator is now open. You see milk, tomato, cheese, and carrots. \\
        2 & \texttt{take tomato from refrigerator} & You are now holding the tomato. \\
        3 & \texttt{go to washbasin} & You are at the washbasin. You see running water. \\
        4 & \texttt{wash tomato in washbasin} & The tomato is now clean. \\
        5 & \texttt{go to shelf1 (next to stove)} & You are near the stove. You see shelf1 next to stove. \\
        6 & \texttt{put tomato on shelf1} & The tomato has been placed on shelf1 next to the stove. \\
        \hline
    \end{tabular}
    \label{fig:alfworld_trajectory}
\end{figure*}

\end{document}